\documentclass{article}

\usepackage[final,nonatbib]{neurips_2020}

\usepackage[utf8]{inputenc} 
\usepackage[T1]{fontenc}    
\usepackage{hyperref} 
\usepackage{url} 
\usepackage{booktabs} 
\usepackage{amsfonts} 
\usepackage{nicefrac} 
\usepackage{microtype} 
\usepackage{wrapfig}
\usepackage{multirow}

\usepackage{subcaption}
\usepackage{graphicx}
\usepackage{url}

\usepackage{multicol}

\title{\our{}: Representing 3D Objects as Surfaces}

\author{%
  Przemys\l{}aw Spurek$^*$\\
 Jagiellonian University, \\
 Kraków, Poland\\
   \texttt{przemyslaw.spurek@uj.edu.pl} \\
  \And
  Maciej Zięba\thanks{Equal contribution} \\
  Wrocław University of \\
  Science and Technology, \\
  Wrocław, Poland
  \AND
  Jacek Tabor \\
 Jagiellonian University, \\
 Kraków, Poland\\
  \And
  Tomasz Trzciński  \\
  Warsaw University of \\
  Technology, \\
  Warsaw, Poland
}

\usepackage{amsmath,amsfonts,amssymb,amsthm,dsfont,stmaryrd}
\usepackage[utf8]{inputenc}
\usepackage{hyperref}

\usepackage{graphicx}

\usepackage{amssymb}

\usepackage{color}

\def\P{\mathcal{P}}
\def\Q{\mathcal{Q}}
\def\X{\mathcal{X}}
\def\Z{\mathcal{Z}}
\def\E{\mathcal{E}}
\def\C{\mathcal{C}}

\def\D{\mathcal{D}}
\def\L{\mathcal{L}}

\def\Tr{\mathrm{Tr}}

\def\R{\mathbb{R}}
\def\EE{\mathbb{E}}

\def\our{HyperFlow}
\def\ourdens{Spherical Log-Normal}

\newtheorem{theorem}{Theorem}[section]

\theoremstyle{remark}

\begin{document}

\maketitle

\begin{abstract}

In this work, we present \our{} - a novel generative model that leverages hypernetworks to create continuous 3D object representations in a form of lightweight surfaces (meshes), directly out of point clouds. Efficient object representations are essential for many computer vision applications, including robotic manipulation and autonomous driving. However, creating those representations is often cumbersome, because it requires processing unordered sets of point clouds. Therefore, it is either computationally expensive, due to additional optimization constraints such as permutation invariance, or leads to quantization losses introduced by binning point clouds into discrete voxels. Inspired by mesh-based representations of objects used in computer graphics, we postulate a fundamentally different approach and represent 3D objects \emph{as a family of surfaces}. To that end, we devise a generative model that uses a hypernetwork to return the weights of a Continuous Normalizing Flows (CNF) target network. The goal of this target network is to map points from a probability distribution into a 3D mesh. To avoid numerical instability of the CNF on compact support distributions, we propose a new \ourdens{} function which models density of 3D points around object surfaces mimicking noise introduced by 3D capturing devices. 
As a result, we obtain continuous mesh-based object representations that yield better qualitative results than competing approaches, while reducing training time by over an order of magnitude.

\end{abstract}


\section{Introduction}


Representing 3D objects efficiently is a prerequisite for a multitude of contemporary computer vision and machine learning applications, including robotic manipulation~\cite{kehoe2015survey} and autonomous driving~\cite{yang2018pixor}. 3D registration devices used currently to create those representations, such as LIDARs and depth cameras, sample object surfaces and output a set of 3D points called a \emph{point cloud}. 

Processing point clouds poses several challenges. 
First of all, the size of the point cloud can vary between objects and processing variable-size inputs is cumbersome for contemporary neural networks used in practical applications. Although one can subsample or upsample point clouds, it requires additional processing steps, continuous signed distance functions~\cite{park2019deepsdf} or even separate models~\cite{yifan2019patch,yu2018pu}. Other solutions to that problem rely on discretizing 3D space into regular 3D voxel grids~\cite{wu20153d,wu2016learning}, collections of images~\cite{su2015multi} or occupancy grids~\cite{ji20123d,maturana2015voxnet}. These approaches, however, increase the memory footprint of object representations and lead to quantization losses. Secondly, processing point clouds with neural networks is challenging due to the lack of ordering within sets of 3D points. More precisely, permuting the points in the cloud can lead to inconsistent outputs. DeepSets~\cite{zaheer2017deep} and PointNet~\cite{qi2017pointnet,qi2017pointnet++} address this problem by including permutation invariant layers in neural network architectures. 
Nonetheless, the same modifications cannot be used when the task requires a model to produce outputs of various sizes, {\it e.g.} in the case of point cloud reconstruction tasks.

More recent methods that create representations of 3D objects from variable-size unordered point clouds rely on generative neural networks that treat point clouds as a sample from a 3D probability distribution~\cite{yang2019pointflow,stypulkowski2019conditional,spurek2020hypernetwork}. PointFlow~\cite{yang2019pointflow} returns probability distributions of the 3D object point cloud, instead of an exact set of points. Its main limitation, however, is a computationally expensive training process caused by conditioning the Continuous Normalizing Flow (CNF) module~\cite{grathwohl2018ffjord} of the network with the autoencoder latent space. As a consequence, PointFlow models require a significant number of parameters which results in a high memory footprint of the model and long training procedure. 
To reduce this burden and simplify the model, HyperCloud~\cite{spurek2020hypernetwork} uses a hypernetwork, instead of a CNF module as in PointFlow, to return weights of a fully-connected \emph{target network} that maps a uniform distribution on a 3D ball to a 3D point cloud. Although the simplicity of this approach leads to increased efficiency of HyperCloud, the quantitative results obtained by the model are inferior to those of PointFlow, mostly because conventional fully-connected neural networks are not capable of modeling complex 3D point cloud structures. Even though using more sophisticated CNF as a target network could address this shortcoming, the formulation of HyperCloud does not allow sampling from non-compact support prior, required by the Continuous Normalizing Flow (CNF) to work.




\begin{figure}
\begin{center}
\includegraphics[height=5.4cm]{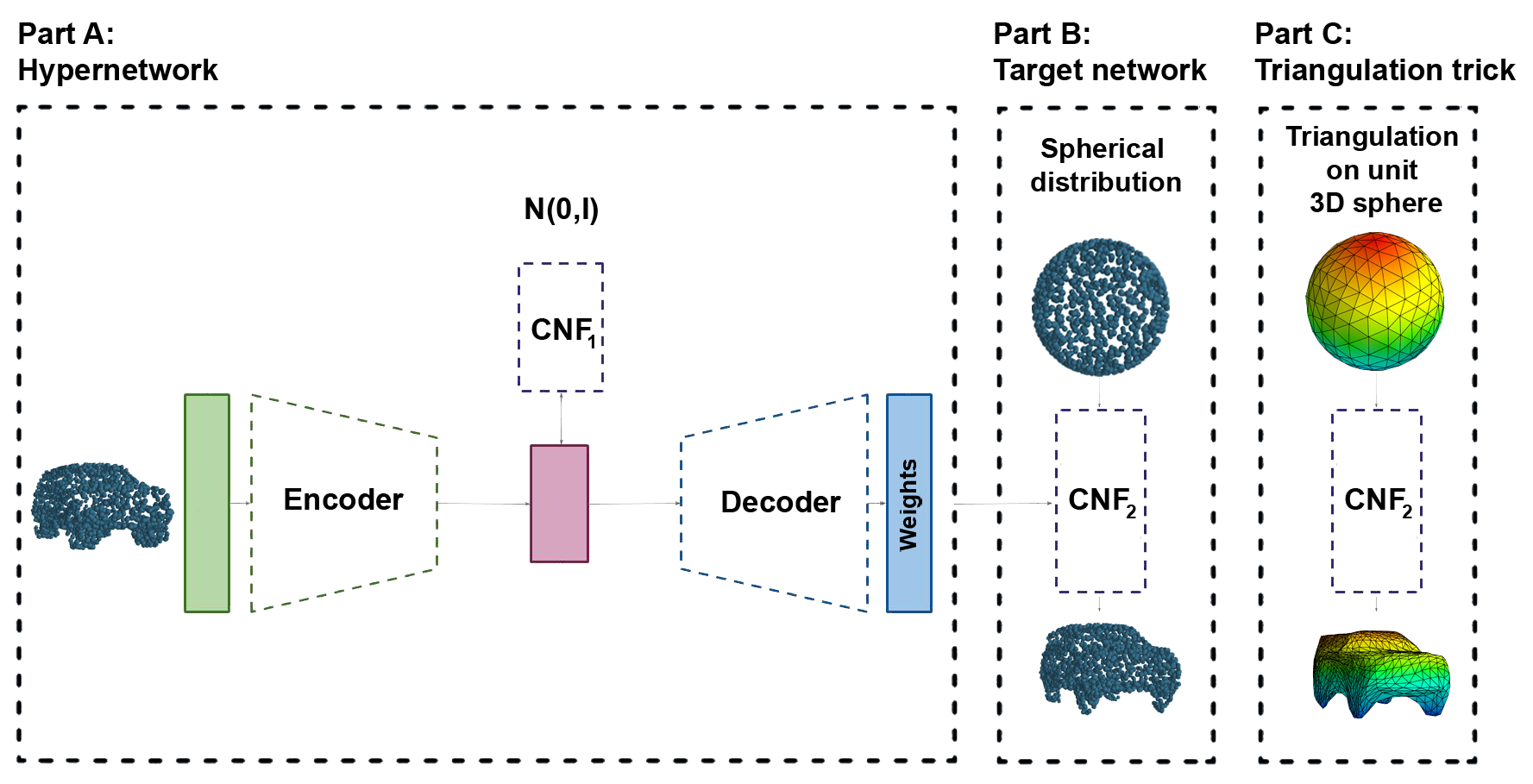}
\end{center} 
\caption{\our{} method leverages a hypernetwork architecture to take a 3D point cloud as an input and return parameters of the Continuous Normalizing Flow (CNF) target network~{\bf(Part~A)}. To represent 3D objects as \emph{families of surfaces}, we use CNF to parametrize density of point clouds around surfaces with non-compact support distribution called \ourdens{}~{\bf(Part~B)}. Using this parametrization in the hypernetwork configuration, we can obtain high-quality point cloud reconstructions as well as 3D object meshes{~\bf(Part~C)}, at a fraction of the training cost required by the vanilla CNF model and with a significantly lower memory footprint.
}
\vspace{-0.5cm}
\label{fig:idea}
\end{figure}

In this paper, we take a fundamentally different approach to representing 3D objects and, inspired by mesh triangulation methods used in computer graphics~\cite{edelsbrunner_2000}, we model objects as \emph{families of surfaces}. More specifically, we consider a point cloud as a sample from a distribution on object surfaces with additive noise introduced by a registration device, such as LIDAR. To model this distribution, we propose a new \ourdens{} function which mimics the topology of 3D objects and provides non-compact support. This, in turn, enables effective utilization of a CNF model as a part of a hypernetwork, instead of a fully-connected neural network as done in HyperCloud~\cite{spurek2020hypernetwork}. 

The resulting generative model we introduce in this work, dubbed \our{}\footnote{The code is available \url{https://github.com/maciejzieba/HyperFlow}.}, produces state-of-the-art generative results both for point clouds and mesh representations. 
Because we rely on a hypernetwork instead of conditioning a CNF with the autoencoder latent space, our model uses far fewer parameters of the CNF function. 
As a result, we reduce the training time and corresponding memory footprint of the model by over an order of magnitude with respect to the competing PointFlow.

Our contributions can be summarized as follows:
\begin{itemize}
    \item We introduce a new \our{} generative network that models 3D objects as \emph{families of surfaces} and allows to build state-of-the-art point cloud representations that can be transformed into 3D meshes by leveraging generative properties of a target network. 
    \item  We propose a new \ourdens{} distribution which models a point cloud density with non-compact support and, hence, can be effectively used by a CNF model.
    \item To the best of our knowledge, our work is the first approach to train a CNF as a target network which reduces its training time and memory footprint by over an order of magnitude, while preserving state-of-the-art generative capabilities.
\end{itemize}

\section{\ourdens{} distribution and the triangulation trick }\label{sec:density}

In this section, we introduce a \ourdens{}  distribution that models density of point clouds around surfaces of 3D object and show how it can be used to generate meshes via the so-called \emph{triangulation trick}. 

\begin{wrapfigure}{r}{0.5\textwidth}
\begin{center}
 \includegraphics[height=3.2cm]{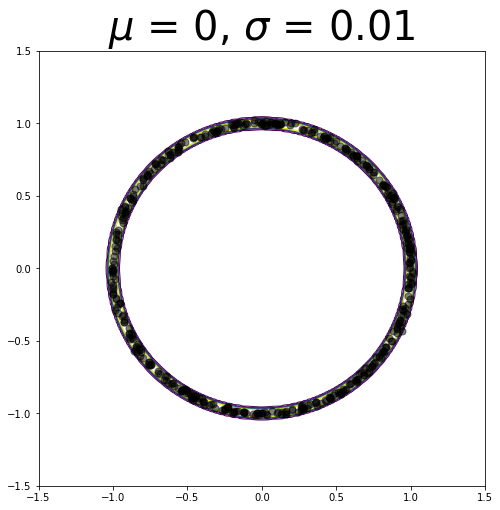}
 \includegraphics[height=3.2cm]{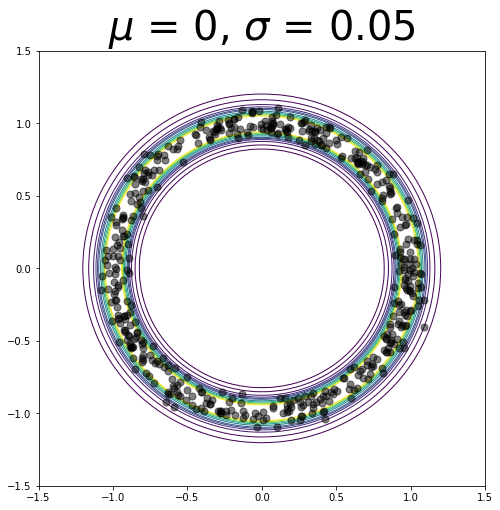}
\end{center} 
\caption{Level sets and samples from \ourdens{} distribution with different parameters~$\sigma$ and $\mu=0$. Since \ourdens{} distribution does not have a compact support, it can be used in flow-based architectures.}
\label{fig:1} 
\end{wrapfigure}

Since our approach relies on flow-based models, a density distribution has to fulfil several conditions to be used in practice. 
First of all, flow-based methods cannot be trained on probability distributions with compact support. 
For instance, it is not possible to train a flow-based model on a 3D ball, as proposed in HyperCloud~\cite{spurek2020hypernetwork}, since computing the log-likelihood cost function used in flows would return infinity for this distribution. As a result, the model does not converge due to numerical instability. Secondly, we would like to model the probability distribution of the surface (mesh representation), which is two-dimensional (the border of a 3D object). Therefore, a Gaussian distribution in $\R^3$ is not a good choice, since it models only elements in 3D.
Finally, the density distribution should be topologically coherent with the density of the modeled object. More precisely, because of the way registration devices sample space around object surfaces, point clouds are populated with the highest density around object edges and missing points within object structure. Modeling this density with a distribution that does not allow discontinuities is infeasible as per Theorem~\ref{th:a}~\cite{theorem1}.

\begin{theorem}\label{th:a}
There is no continuous invertible map between the 3-ball and the 2-sphere that respects the boundary.
\end{theorem}


For modeling the surface of an object with a continuous, invertible map, one shall consider the topology of the object \cite{rezende2015variational,grathwohl2018ffjord,behrmann2018invertible}.
To learn a transformation that is continuous, invertible and provides results close to object boundary, one has to choose a prior that is topologically similar to the expected point cloud, {\it i.e.} has the same number of discontinuities\footnote{Continuous normalizing flows (FFJORD \cite{grathwohl2018ffjord}) are able to approximate discontinuous density functions. This, however, remains insufficient to model high-quality 3D point clouds while generating continuous parametrization of object surfaces. 
Consequently, in our approach, we propose a density distribution without compact support and with a single discontinuity, which corresponds to topology of 3D objects represented with point clouds.}. Therefore, we construct a probability distribution on a sphere without compact support.

\paragraph{\ourdens{} distribution on $\R^n$.} 

A probability distribution on a sphere in $\R^n$ can by constructed by using one-dimensional density distribution, which takes only positive real values
$f:\R_+ \to \R_+.$
In such case, we can define spherical density distribution as:
\begin{equation}
f_{n}: \R^n \ni x \to \frac{1}{\mathrm{vol}(S_{n-1}) \|x\|^{n-1}}f(\|x\|),
\end{equation}

where $\mathrm{vol}(S_{n-1})$ is a surface area of a $n$-dimensional unitary sphere and $f$ is a one-dimensional density, which takes only positive real values.
We use one-dimensional density distribution $f:\R_+ \to \R_+$ along radius of unit sphere in all directions.
In our model, we use a Log-normal distribution 
$
f(r)=\frac{1}{r}\cdot \frac {1}{\sigma \sqrt{2\pi}}
\exp \left(-{\frac {(\log r-\mu )^{2}}{2\sigma ^{2}}}\right)
$
that is a continuous probability distribution of a random variable, whose logarithm is normally distributed and, hence, provides a non-compact support.

\paragraph{\ourdens{} distribution in $\R^3$.}
To develop an intuition behind the proposed distribution, we start with a simple visualization in $\R^2$. 
Fig.~\ref{fig:1} shows level sets and sample from \ourdens{} distribution with different parameters $\sigma$. 
\ourdens{} distribution does not have a compact support and can therefore be used in a flow-based architecture. Furthermore, we can force the distribution to concentrate as close as possible to a 2D sphere boundaries. 

In $\R^3$, our \ourdens{} distribution is defined as:
\begin{equation}
f_{3}(x)=\frac{1}{2(2\pi)^{3/2}\sigma\|x\|^3}
\exp \left(-\frac{(\log\|x\|-\mu)^2}{2\sigma^2}\right).
\end{equation}

In order to use our distribution in a flow-based model, we need to compute its log-likelihood function:
\begin{equation}
\log f_{3}(x)=-\log(2(2\pi)^{3/2})-\log \sigma -3\log \|x\| - \frac{1}{2\sigma^2} (\log\|x\|-\mu)^2.
\end{equation}

Finally, sampling elements from our \ourdens{} distribution can be done by following a simple procedure. First sample $r$ from one-dimensional Gaussian $N(0,1)$ then sample $x$ from $n$-dimensional Gaussian $N(0,I)$. Sample form \ourdens{} we obtain by the following equation:
    $
    \exp(\mu+\sigma \cdot r) \cdot \frac{x}{\|x\|}.
    $


We avoid numerical instabilities of training by applying a straightforward strategy to find the right values of $\sigma$ parameter: we start with an arbitrary large value of $\sigma$ and reduce it linearly during training.

\begin{figure}[t!]
\begin{center}
\includegraphics[height=7.5cm]{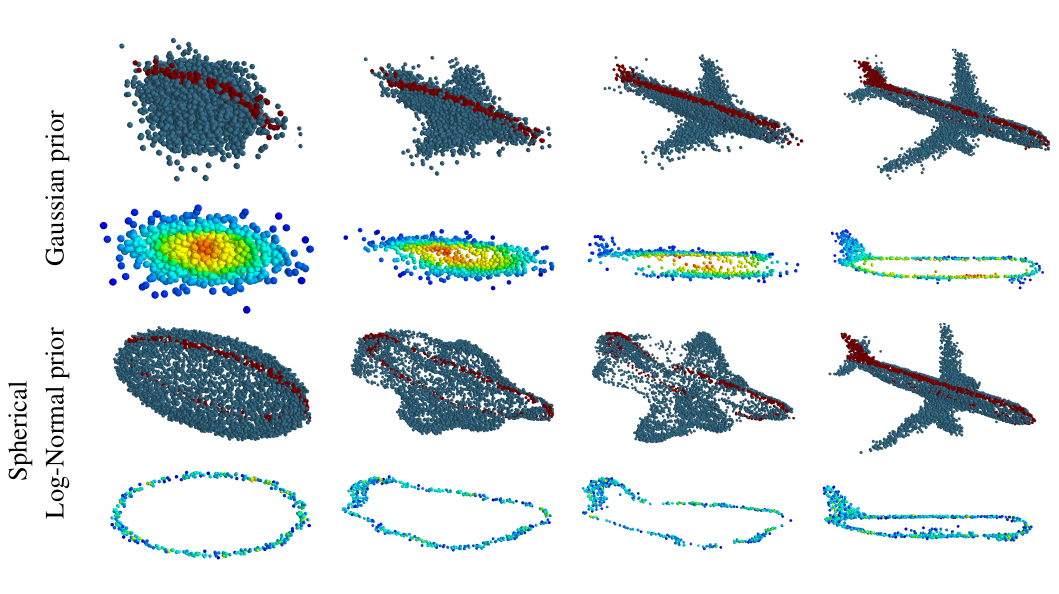}
\end{center} 
\caption{We compare how the prior density is modified for the model with Gaussian prior {\bf(upper two rows)} and \ourdens{} {\bf(bottom two rows)}. In the first and third row we show how the flow model transforms the original density into the target dataset. The second and fourth row show the cross-sections along the plane depicted by red points. For the Gaussian, target space points are not distributed evenly across the object (a central part of Gaussian distribution is transformed into the bottom of the plane, while its tails are used to model wing tips). For the \ourdens{}, target space points are distributed evenly, across the object, showcasing that our approach truly models the distribution of the points along object surfaces.} 
\label{fig:slice}
\end{figure}

\paragraph{Triangulation trick}

To model 3D object surfaces as meshes using \our{} generative model, we need to investigate the relationship between point clouds and object surfaces. In principle, a point cloud representing a 3D object can be considered a set of samples located on the surface of the object with additive noise introduced by a registration device. We use \ourdens{} to model this distribution with peak density around object surfaces (in 2D, around circle edges, in 3D close to the surface of the sphere) and limited by the radius of the distribution. %
Once we obtain a parametrized distribution of a point cloud which models object surface together with a registration noise, we can produce a mesh with a simple operation which we call the  \emph{triangulation trick}.  

The triangulation trick involves transferring vertices of a sphere mesh through a target network the same way as 3D points, as shown in Part C of Fig~\ref{fig:idea}. Since the target network transforms a sample from a \ourdens{} distribution into a 3D point cloud, when we feed it with a sphere triangulation, it outputs a mesh. In fact, when we substitute samples from \ourdens{} distribution with sphere vertices, we effectively assume minimal registration noise. Processing vertices by the target network pre-trained on point clouds allows us to directly generate denoised mesh representation of object surfaces and obtain a high-quality 3D object rendering. The generative character of our \our{} model enables construction of the entire mesh by processing only vertices with a target network, without the need for information about the connections between them, as done in traditional rendering methods. 

Fig.~\ref{fig:slice} presents reconstructions obtained using Gaussian and \ourdens{} distributions. We look at the cross-sections of the reconstructions to observe the main differences on how the input distribution is transformed into a final model by a target network. 
For the Gaussian distribution, its tails are transformed into object details, such as wing tips and airplane rear aileron. Therefore, we cannot claim that the peak density models surfaces of the object, while its tails model the registration noise. 
For \ourdens{}, its distribution tails are spread along object surfaces, modeling registration noise. This allows us to produce the final mesh through the triangulation trick, effectively denoising 3D mesh-based object representation and yielding high-quality results, as shown in Fig.~\ref{fig:ec_mesh}. 

\begin{figure}[t!]
\begin{center} 

 \includegraphics[height=2.8cm]{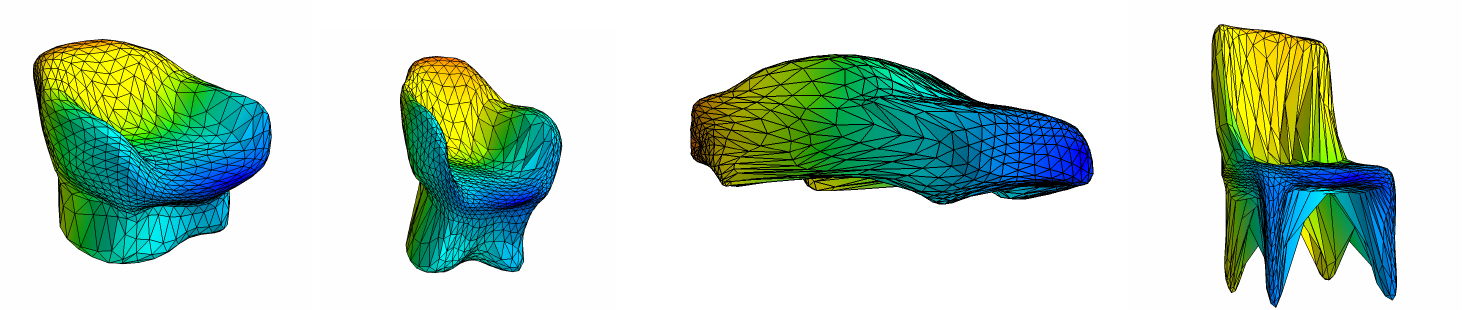}

\end{center} \vspace{-0.3cm}
\caption{Mesh representations generated by our \our{} method. Contrary to the existing methods that return point cloud representations sparsely distributed in 3D space, our approach allows to create a continuous 3D object representation in the form of high-quality meshes.}
\vspace{-0.3cm}
\label{fig:ec_mesh} 
\end{figure}  

\section{\our{}: hypernetwork and Continuous Normalizing Flows for generating 3D point clouds}
\label{sec:method}

In this section, we present our \our{} model that leverages a hypernetwork framework to train a Continuous Normalizing Flow~\cite{grathwohl2018ffjord} target network and generate 3D point clouds together with its mesh-based representation. Since \our{} encompasses previously introduced  autoencoder-based PointFlow~\cite{yang2019pointflow} with conditioned continuous normalizing flow modules, and HyperCloud method~\cite{ha2016hypernetworks}, that also leverages hypernetworks, we briefly describe these two approaches before presenting ours.

\paragraph{Autoencoder-based generative model for 3D Point Clouds}

Let us first present the autoencoder architecture. The basic aim of autoencoder is to transport the data through a typically, but not necessarily, lower dimensional latent space $\Z = \R^D$ while minimizing the reconstruction error.  Thus, we search for an encoder $\E:\X \to \Z$ and decoder $\D:\Z \to \X$ functions, which minimizes the reconstruction error. 
In the Autoencoder-based generative model we additionally ensure that the data transported to the latent comes from the prior distribution (typically Gaussian one)~\cite{kingma2013auto,tolstikhin2017wasserstein,tabor2018cramer}.
\paragraph{Continuous normalizing flow}

Generative models are one of the fastest growing areas of deep learning. Variational Autoencoders (VAE)~\cite{kingma2013auto} and Generative Adversarial Networks (GAN)~\cite{goodfellow2014generative} are the most popular approaches. Another model gained popularity -- Normalizing Flow (NF) \cite{rezende2015variational}. A flow-based generative model is constructed by a sequence of invertible transformations. Unlike the other two methods mentioned previously, the model explicitly learns the data distribution  and therefore the loss function is simply the negative log-likelihood. 

Normalizing Flow (NF) \cite{rezende2015variational} is able to model complex probability distributions. A normalizing flow transforms a simple prior distribution (usually Gaussian one) $\P(Y)$ into a complex one (represented by data distribution $X$) by applying a sequence of invertible transformation functions: $f_1, \ldots , f_n : Y \to X$. Flowing through a chain of transformations
$x = F(y) = f_n \circ f_{n-1} \circ \ldots  \circ f_{1}(y), $
we obtain a probability distribution of the final target variable.

Then the probability density of the
output variable is given by the change of variables formula:
\begin{equation}
\log P(x) = \log P(y) - \sum_{k=1}^n \log \left| \det \frac{\partial f_k}{\partial y_{k-1}} \right|,
\end{equation}
where $y$ can be computed from $x$ using the inverse flow:
$y = f^{-1}_1 \circ f^{-1}_{2} \circ \ldots  \circ f^{-1}_{n}(x).$
In such framework, both the inverse map and the determinant of the Jacobian should be computable.

The continuous normalizing flow \cite{chen2018neural} is a modification of the above approach, where instead of a discrete sequence of iterations we allow the transformation to be defined by a solution to a differential equation
$ 
\frac{ \partial y(t)}{ \partial t} = f(y(t), t),
$
where $f$ is a neural network that
has an unrestricted architecture.  Continuous Normalizing Flows (CNF ) $F_{\theta}: Y \ni y \to x \in X$ 
is a solution of differential equations with the initial value problem $y(t_0) = x$, $\frac{\partial y(t)}{\partial t} =f_{\theta}(y(t), t)$. In such a case we have
\begin{equation}
F( y ) = F_{\theta}( y(t_0) ) =  y(t_0) + \int^{t_1}_{t_0} f_{\theta}(y(t), t) dt, \mbox{ and }
F_{\theta}^{-1}(x) = x + \int_{t_1}^{t_0} f_{\theta}(y(t), t)dt,
\end{equation}

where $f_{\theta}$ defines the continuous-time dynamics of the flow
$F_{\theta}$ and $y(t_1) = x$.

The log probability cost function with prior distribution with density $g$ can be computed by:
\begin{equation}\label{eq:flow}
\C_F(X; g, \theta) = \sum_{x \in X}   \log g( F_{\theta}^{-1}(x) ) - \int^{t_1}_{t_0} \Tr \left( \frac{\partial f_{\theta}}{\partial y(t)} \right) dt.
\end{equation}

In PointFlow~\cite{yang2019pointflow} authors show that CNF can be used for modeling 3D objects. Instead of directly parametrizing the distribution of points in a shape (fixed size 3D point cloud), PointFlow models this distribution as an invertible parameterized transformation of 3D points from a prior distribution (e.g., a 3D Gaussian). Intuitively, under this model, generating points for a given shape involves sampling points from a generic Gaussian prior, and then moving them according to this parameterized transformation to their new location in the target shape.

\paragraph{Hypernetwork}

Hypernetworks, introduced in \cite{ha2016hypernetworks}, are defined as neural models that generate weights for a separate target network solving a specific task.
Making an analogy between hypernetworks and generative models, the authors of \cite{sheikh2017stochastic}, use this mechanism to generate a diverse set of target networks approximating the same function. Hypernetworks can also be used for functional representations of images \cite{klocek2019hypernetwork}. 

In the case of generating 3D point clouds, objects are represented by a neural network. Autoencoder based architecture "produces" the neural network which transforms prior distribution into elements from a point cloud. In HyperCloud \cite{spurek2020hypernetwork} autoencoder based architecture takes as an input point cloud and directly produces weights to another neural network, which models elements from a 3D object.  



\paragraph{\our{}}

In this section, we present details of our novel model dubbed \our{}\footnote{We make our implementation available at \url{https://github.com/maciejzieba/HyperFlow}} which encompasses and extends prior works by training continuous normalizing flow modules to model 3D point cloud distributions with a hypernetwork framework. Our model is inspired by a Variational Autoencoder (VAE)~\cite{kingma2013auto,rezende2014stochastic} framework that allows learning $\P(X)$ from
a dataset of observations of $X$. VAE models data distribution via a latent variable $z$ with a prior distribution $\P_{\psi}(z)$, and a decoder $\P_{\theta}(X|z)$ which reconstructs the distribution of $X$ condition on a given $z$. The model is trained together with an encoder $\Q_{\phi}(z|X)$ by minimizing the lower bound on the log-likelihood of the observations (ELBO).

Instead of using a Gaussian prior over shape representations as done in~\cite{yang2019pointflow}, we add another CNF to model a learnable prior $\P_{\psi}$. The corresponding ELBO cost function can be rewritten after~\cite{yang2019pointflow} as: 

\begin{equation}\label{eq:elbo}
\L(X; \phi, \psi, \theta) = \underbrace{\EE_{\Q_{\phi}(z|X)}  \left[ \log \P_{\theta}(X|z) \right]}_{\text{Reconstruction}}   + 
\underbrace{\EE_{\Q_{\phi}(z|X)} \left[ \log \P_{\psi}(z) \right] + H\left[ \Q_{\phi}(z|X) \right]}_{\text{Latent representation}},
\end{equation}
where $H$ is the entropy and $\P_{\psi}(z)$ is the prior distribution
with trainable parameters $\psi$.

We propose to adapt the above cost function to a hypernetwork framework. We therefore introduce our \our{} model that consists of two main parts, as shown in Fig.~\ref{fig:idea}. The first one is a hypernetwork that outputs weights (Fig.~\ref{fig:idea} Part A) of another neural network. The second one is a target network (Fig.~\ref{fig:idea} Part B) which models the distribution of elements on the surface of a 3D object.  
Using autoencoder terminology, we define three elements: an encoder, a decoder and a prior distribution.
The encoder $\E_{\phi}:\X \to \Z$  can reduce data dimensionality by mapping it to a lower-dimensional latent space $\Z \subseteq \R^D$.
We follow~\cite{achlioptas2017learning} and use a simple permutation-invariant encoder to predict $\E_{\phi}$.

We use $\P_{\Z}$ over shape representations proposed by PointFlow~\cite{yang2019pointflow}. 
The assumed probability distribution on the latent pace
can be more complex than  the commonly used $N(0,I)$ and not given in an explicit form. In such a framework, we use an additional continuous normalizing flow $G_{\psi}$, which transfers latent space into a Gaussian prior. 
Finally, we propose to use a decoder that returns weights of the target network~$\D_{\theta}: \Z \ni z \to \Theta$, instead of 3D points as done in~\cite{yang2019pointflow,stypulkowski2019conditional}. The resulting hypernetwork contains an encoder $\E_{\phi}$, a decoder $\D_{\theta}$  and a flow $ G_{\psi} $ (Fig.~\ref{fig:idea} Part A).

The hypernetwork takes as an input a point-cloud $X \subset \R^3$ and returns weights $\Theta$ to  $f_{\Theta}$ that defines the continuous-time dynamics of the flow $F_{\Theta}$.
CNF takes an element from the prior distribution $\P$ and transfers it to an element on the surface of the object, see Part B: target network in Fig.~\ref{fig:idea}. 
In~our~work, we use a  Free-form Jacobian of
Reversible Dynamics (FFJORD) \cite{grathwohl2018ffjord}  and transformation between \ourdens{} distribution and the 3D object. 
As presented in~Sec.~\ref{sec:density} this choice of distribution function allows one to create a continuous mesh representation with the triangulation trick.


The cost function of \our{} consists of two parts. The first one correspond to hypernetwork. This part of the architecture is similar to PointFlow. The second one is a cost function of CNF corresponding to target network.
The final cost function of our~\our{} model can be calculated using  Eq.~(\ref{eq:elbo}):
$$
\C_{H}(X; \theta, \phi, \psi) = \underbrace{\C_F(X; f_{3}, \Theta_{\theta, \phi, \psi})}_{ \text{ Target network cost function } } + \underbrace{\C_G(\E_{\phi}(X); N(0,1), \psi) + H(\E_{\phi}(X)) }_{ \text{ Hypernetwork cost function } },
$$
where $H$ is the entropy function, $\C_F$ is a CNF cost function between point cloud $X$ and \ourdens{}  density $f_{3}$ and $\C_G$ is a CNF cost function between latent representation $\E_{\phi}(X)$ and a Gaussian prior.

\begin{table}
\scalebox{.95}{
\begin{tabular}{@{\hspace{0.0em}}c@{\hspace{0.1em}}|@{\hspace{0.1em}}l@{\hspace{0.4em}}l@{\hspace{0.4em}}l@{\hspace{0.2em}}l@{\hspace{0.4em}}l@{\hspace{0.1em}}|@{\hspace{0.1em}}l@{\hspace{0.4em}}l@{\hspace{0.4em}}l@{\hspace{0.4em}}l@{\hspace{0.4em}}l@{\hspace{0.1em}}|l@{\hspace{0.1em}}l@{\hspace{0.4em}}l@{\hspace{0.2em}}l@{\hspace{0.4em}}l@{}}
\hline
\hline
& \multicolumn{5}{c|}{\textit{Airplane}} & \multicolumn{5}{c|}{\textit{Chair}} & \multicolumn{5}{c}{\textit{Car}} 
\\ \hline
\multirow{2}{*}{Method} 
& \multirow{2}{*}{JSD} & \multicolumn{2}{c}{MMD} & \multicolumn{2}{c|}{COV}
& \multirow{2}{*}{JSD} & \multicolumn{2}{c}{MMD} & \multicolumn{2}{c|}{COV}
& \multirow{2}{*}{JSD} & \multicolumn{2}{c}{MMD} & \multicolumn{2}{c}{COV}\\  
&  & CD & EMD & CD & EMD 
&  & CD & EMD & CD & EMD
&  & CD & EMD & CD & EMD\\ 
\hline 

 l-GAN    & \bf 3.61 & 0.269 & 3.29 & \bf 47.90 &  50.62 & 2.27 & 2.61  &  7.85 & 40.79 & 41.69 & 2.21 & 1.48  & 5.43 & 39.20 & 39.77 \\
 PC-GAN  & 4.63 & 0.287 & 3.57 & 36.46 & 40.94 & 3.90 & 2.75  & 8.20 & 36.50 & 38.98 & 5.85 & 1.12  & 5.83 & 23.56 & 30.29\\
 PointFlow     & 4.92 & \bf 0.217 & 3.24 & 46.91 & 48.40 &  1.74 &  2.42  & 7.87 & \bf 46.83 & \bf 46.98 & \bf 0.87 & \bf 0.91  & \bf 5.22 &  44.03 &  46.59\\
 HyperCloud & 4.84 & 0.266 & 3.28 & 39.75 & 43.70 & 2.73 & 2.56  & \bf 7.84 & 41.54 & 46.67 & 3.09 & 1.07  & 5.38 & 40.05 & 40.05   \\ 
 \our{} &  5.39    &   0.226     &  \bf 3.16   &    46.66   &    \bf 51.60 & \bf 1.50   &  \bf 2.30    &   8.01  & 44.71 &  46.37 &  1.07   &    1.14  &  5.30   & \bf 45.74 &  \bf 47.44 \\
\hline   
\end{tabular}
}
\caption{Generation results. MMD-CD scores are multiplied by
$10^3$; MMD-EMD scores and JSDs are multiplied by $10^2$. }
\label{tab:gen_results}
\end{table}

\section{Experiments}
\label{sec:experiments}

In this section, we present the evaluation of our model against the competing methods on two tasks: 3D point clouds generation and 3D mesh generation. Furthermore, we test the efficiency of our approach in terms of training time and memory footprint. All experiments are done on a stationary unit with a Nvidia GeForce GTX 1080 GPU. If not stated otherwise, default parameters are used. 

\paragraph{Generating 3D point clouds}

We compare the generative capabilities with competing approaches: latent-GAN~\cite{achlioptas2017learning}, PC-GAN~\cite{li2018point}, PointFlow~\cite{yang2019pointflow}, HyperCloud~\cite{spurek2020hypernetwork}. We follow the evaluation protocol of~\cite{yang2019pointflow} and train each model using point clouds from one of the three categories in the ShapeNet dataset~\cite{shapenet}: \emph{airplane},
\emph{chair}, and \emph{car}. 
Tab.~\ref{tab:gen_results} presents the results and shows that \our{} obtains comparable or superior generative results to the state-of-the-art PointFlow method.

\paragraph{Generating 3D meshes}

The main advantage of our method, when compare to the reference solutions, is the ability to generate high-quality 3D point clouds as well as meshes using the triangulation trick presented in Sec.~\ref{sec:density}. For evaluation of the quality of mesh grid representation, we follow the evaluation protocol of~\cite{spurek2020hypernetwork}. For PointFlow, we use the triangulation trick and create object meshes by feeding the target network a 3D sphere. For HyperCloud and our \our{} method we use a sphere with radius $R=1$. As can be seen in Tab.~\ref{tab:sphere}, PointFlow that uses a Gaussian distribution as a prior provides results inferior to HyperCloud and \our{}, while our \our{} method offers the best performance, thanks to using~\ourdens{} as a prior instead of a compact support distribution function as in HyperCloud. More qualitative mesh results as well as detailed description of metrics used in our experiments can be found in the supplementary material.   
\paragraph{Training time and memory footprint comparison}

\begin{table}
\scalebox{0.92}{
\begin{tabular}{@{\hspace{0.0em}}c@{\hspace{0.1em}}|@{\hspace{0.1em}}l@{\hspace{0.4em}}l@{\hspace{0.4em}}l@{\hspace{0.2em}}l@{\hspace{0.4em}}l@{\hspace{0.1em}}|@{\hspace{0.1em}}l@{\hspace{0.4em}}l@{\hspace{0.4em}}l@{\hspace{0.4em}}l@{\hspace{0.4em}}l@{\hspace{0.1em}}|l@{\hspace{0.1em}}l@{\hspace{0.4em}}l@{\hspace{0.2em}}l@{\hspace{0.4em}}l@{}}
\hline
& \multicolumn{5}{c|}{\textit{Airplane}} & \multicolumn{5}{c|}{\textit{Chair}} & \multicolumn{5}{c}{\textit{Car}} 
\\ \hline
\multirow{2}{*}{Sphere R} 
& \multirow{2}{*}{JSD} & \multicolumn{2}{c}{MMD} & \multicolumn{2}{c|}{COV}
& \multirow{2}{*}{JSD} & \multicolumn{2}{c}{MMD} & \multicolumn{2}{c|}{COV}
& \multirow{2}{*}{JSD} & \multicolumn{2}{c}{MMD} & \multicolumn{2}{c}{COV}\\  
&  & CD & EMD & CD & EMD 
&  & CD & EMD & CD & EMD
&  & CD & EMD & CD & EMD\\ 
\hline 

PointFlow  & & & & &  & & & & &  & & & & &     \\ \hline
R=2.795    & 22.26 & 0.49    & 6.65    & 44.69 & 20.74 &  19.28 & 4.28 & 13.38 & 36.85 & 20.84  & 16.59 & 1.6     & 8.00    & 20.17    & 17.04    \\
R=3.136    & 26.46 & 0.60    & 6.89    & 39.50    & 19.01 & 22.52 & 4.89 & 14.47 & 32.47 & 17.22  & 20.21 & 1.75    & 7.80    & 21.59    & 17.3 \\
R=3.368    & 29.65 & 0.68    & 6.84    & 40.49    & 16.79 &
24.68 & 5.36    & 14.97   & 31.41    & 17.06 & 24.10 & 1.96    & 8.35    & 18.75    & 17.04\\ \hline
HyperCloud    & & & & &  & & & & &  & & & & & \\ \hline
R=1  & 9.51  & 0.45 & 5.29 & 30.60    & 28.88  & 4.32  & \bf 2.81 & 9.32 & 40.33 & 40.63   & \bf 5.20  & \bf 1.11 & 6.54 & 37.21 & 28.40  \\ \hline

\our{}    & & & & &  & & & & &  & & & & & \\ \hline
R=1  & \bf 6.55  & \bf 0.38 &  \bf 3.65 & \bf 40.49 & \bf 48.64 &  \bf 4.26 &  3.33 & \bf 8.27 & \bf 41.99 & \bf 45.32 & 5.77 & 1.39 & \bf 5.91 & 28.40 & \bf 37.21 \\ 

\end{tabular}
}
\caption{The values of quality measures of 3D representations obtained by sampling from sphere of a given radius $R$ for \emph{airplane}, \emph{chair} and \emph{car} shapes. \our{} generates higher quality of point cloud representation than those of PointFlow and HyperCloud. }
\label{tab:sphere}
\end{table}

Fig.~\ref{fig:proc_comparison} displays a comparison between our \our{} method and the competing PointFlow. For a fair comparison we evaluated the architectures used in the previous sections that obtain best quantitative results. The models were trained on the \emph{car} dataset.  Our~\our{} approach leads to a significant reduction in both training time and memory footprint due to a more compact flow architecture enabled by a hypernetwork framework. 


\begin{figure}[t!]
\centering
  \begin{subfigure}{6.8cm}
    \centering \includegraphics[width=\linewidth]{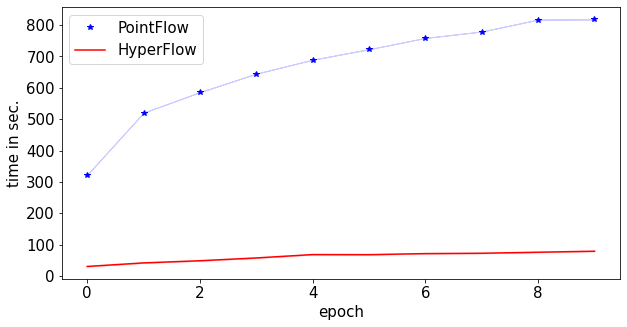}
   \end{subfigure}
     \begin{subfigure}{6.8cm}
         \centering
 \includegraphics[width=\linewidth]{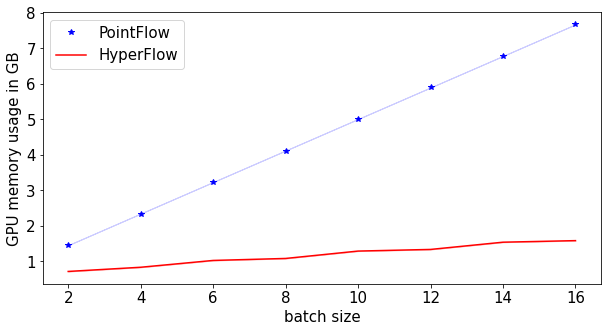}
    \end{subfigure}
\caption{Comparison of training times and GPU memory used by PointFlow and \our{}. Our \our{} method offers over an order of magnitude decrease in both training time and memory.}


\label{fig:proc_comparison}
\end{figure}

\section{Conclusions}

In this work, we introduce a novel \our{} method that uses a hypernetwork to model 3D objects as \emph{families of surfaces} and, hence, allows to  build state-of-the-art point cloud reconstructions and mesh-based object representations. To model a distribution of a point cloud we propose a new \ourdens{} distribution with non-compact support that can be effectively used by a CNF model. Finally, we believe our work is the first approach to train CNF as a target network which reduces training cost and opens new research paths for modeling complex 3D structures, such as indoor scenes.
    
\section*{Broader Impact}

This research can be beneficial for researchers and engineers working in the space of 3D point clouds and related registration devices, such as LIDARs and depth cameras. As such, the proposed methods can be used in the context of autonomous driving and robotics. Further extensions of this work can be beneficial for people with disparities, especially related to sensory disorders, such as shortsightedness or blindness, as 3D capturing devices can effectively extend their way of interacting and perceiving the external world. On the other hand, robotic automation resulting from this work can potentially put at disadvantage people whose livelihoods depend on manual execution of jobs that can be substituted with robotics. In case of system failure, the consequences include problems with handling outputs of registration devices, such as LIDARs and depth cameras. Our method does not leverage any biases in the data.

\section{Supplementary material}

In this supplementary material, we first present the full description of evaluation metrics used in the experiments. We then describe two experiments showing the relationship between Gaussian distribution and \ourdens{} distribution proposed in our work. Finally, we show an extended set of visualizations obtained by \our{}.

\subsection{Description of evaluation metrics}






Following the methodology for evaluating generative fidelity and diversification among samples proposed in \cite{achlioptas2017learning} and \cite{yang2019pointflow}, we use the following evaluation metrics: Jensen-Shannon Divergence, Coverage, Minimum Matching Distance 1-nearest Neighbor Accuracy. 

Jensen-Shannon Divergence (JSD): a measure of the distance between two empirical distributions $P$ and $Q$, defined as:
$$
\begin{array}{c}
JSD(P\|Q) \! = \! \frac{KL(P\|M) + KL(Q\|M)}{2}, \mbox{ where } M \! = \! \frac{P+Q}{2}.
\end{array}
$$

Coverage (COV): a measure of generative capabilities in terms of richness of generated samples from the model. For two point cloud sets $X_1, X_2 \subset \R$ coverage is defined as a fraction of points in $X_2$ that are in the given metric the nearest neighbor to some points in $X_1$.

Minimum Matching Distance (MMD): since COV only takes the closest point clouds into account and does not depend on the distance between the matchings additional metric was introduced. For point cloud sets $X_1$, $X_2$ MMD is a measure of similarity between point clouds in $X_1$ to those in $X_2$.

We examine the generative capabilities of our \our{} model with respect to the existing reference approaches. We strictly follow the methodology presented in \cite{yang2019pointflow}.  We train each model using point clouds from one of the three categories in the ShapeNet dataset: \emph{airplane},
\emph{chair}, and \emph{car}.

\subsection{ Scheduling parameters of \ourdens{} }

In our model we use \ourdens{} density with $m=0$ and $\sigma=0.001$. Using \ourdens{} density with small $\sigma=0.001$ might be unstable since density distributing has small tails, see Fig.~2 (in main paper). At the beginning of training a log-likelihood cost function in some points might be close to zeros (numerically unstable).

Therefore, in the training procedure we start with large $\sigma= 1$ and reduce such parameter to $\sigma= 0.001$. We use linear scheduling. In the case of starting $\sigma_0$ and final value of $\sigma_1$ with $n$ epochs we reduce the parameter by $\Delta \sigma = \frac{\sigma_1-\sigma_0}{n}$ in each epoch.

Our model is approximately 10 times faster than PointFlow (see experimental section in main paper), and can be easily trained on \our{} density. In PointFlow architecture is larger and it is diffitult to train such model on our distribution from scratch.
This process can be accelerated by using pre-trained model on classical Gaussian distribution. 
In such a case we can start from \ourdens{} distribution with parameters $\mu$ and $\sigma$ which approximate Gaussian distribution (see Theorem~\ref{th_nor}). In Fig.~\ref{fig:com_den} we present comparison between samples from Gaussian distribution and \ourdens{} distribution with such parameters.   
Thanks to such solution we can take a model already trained on Gaussian distribution and train it further with our strategy.

\begin{figure}
\begin{center}
 \includegraphics[height=6cm]{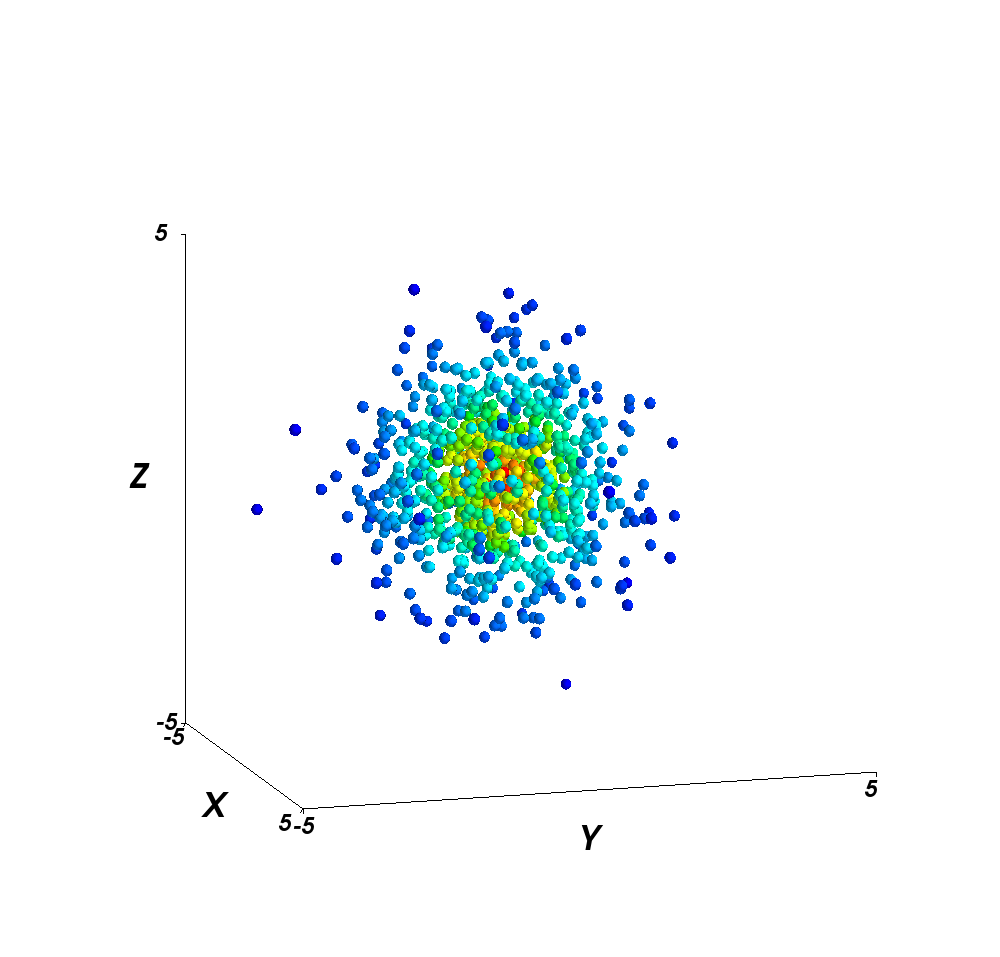}
 \includegraphics[height=6cm]{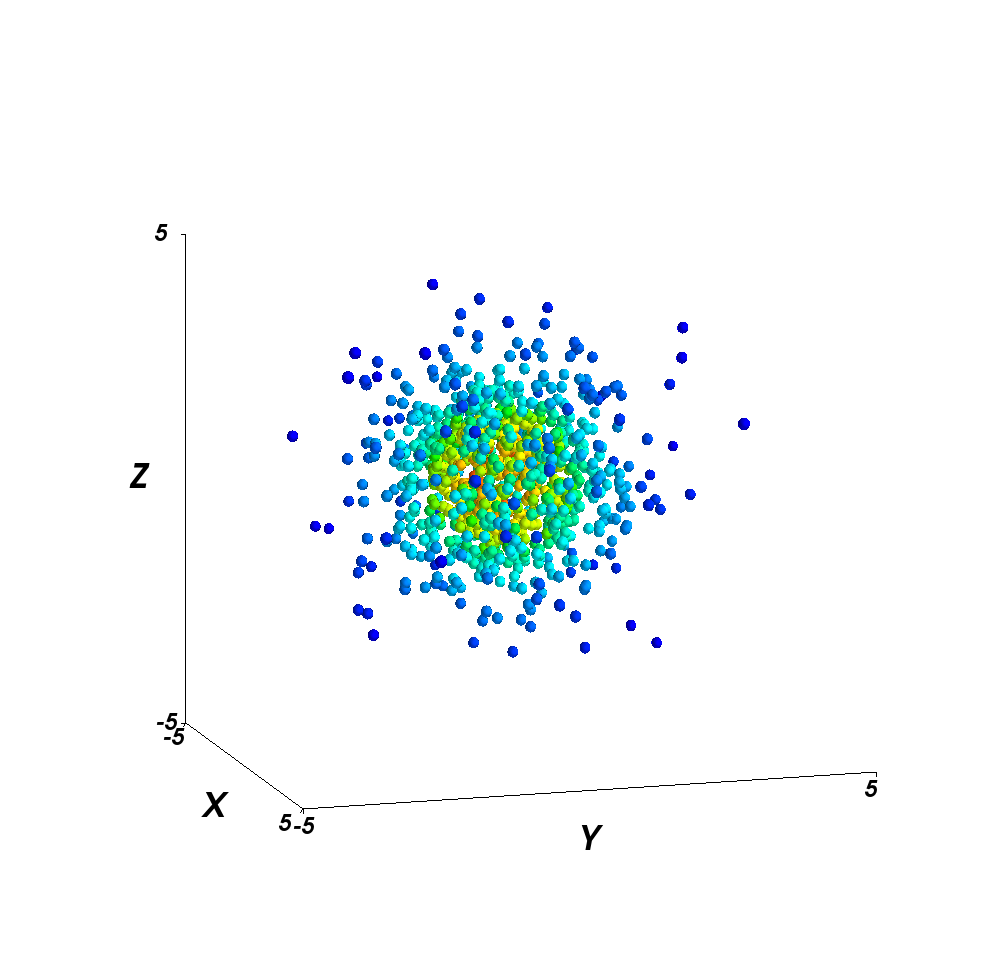}
\end{center} 
\caption{Comparison between samples from a Gaussian distribution {\bf(left)} and \ourdens{} distribution {\bf(right)} which approximates normal distribution (with parameters from Theorem~\ref{th_nor}).}
\label{fig:com_den} 
\end{figure}
\begin{figure}[t!]
\begin{center}
  \includegraphics[height=6.0cm]{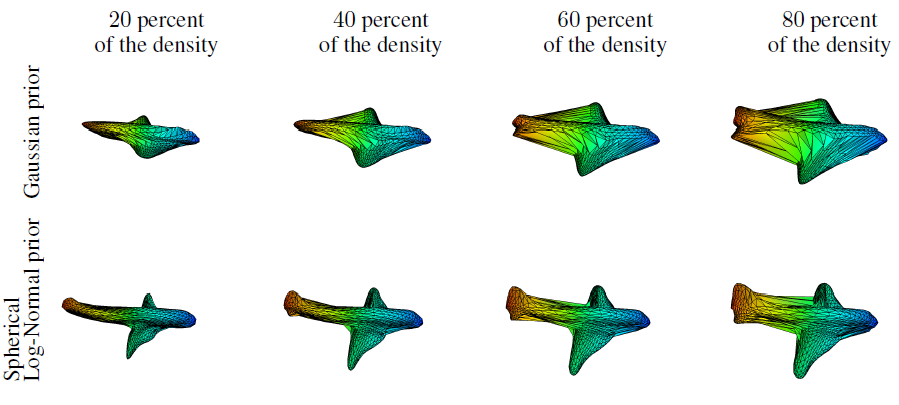} 
\end{center} 
\caption{Object meshes generated for different radii. For the Gaussian prior, the quality and size of the mesh heavily depends on the radius size, while for \ourdens{} the quality and size remains stable across radii sizes.}
\label{mes_diff}
\end{figure}

\begin{figure}
\begin{center}
 \includegraphics[height=8cm]{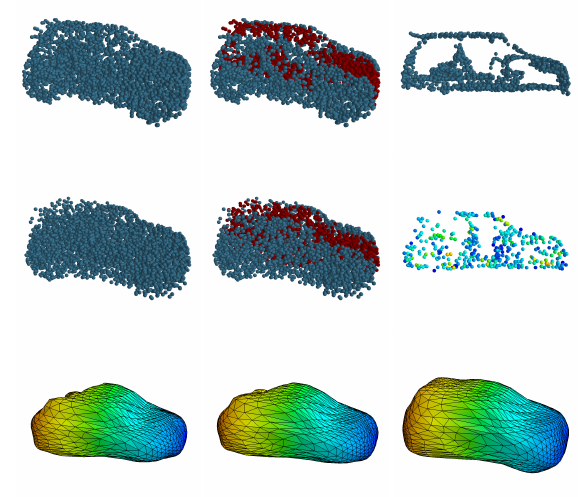}
\end{center} 
\caption{In the first row, we present a car from a data-set which contains elements inside objects. In the second row, we present reconstructions of the object. In the third row, we show meshes generated by radii which contains $40,60$ and $80$ percent of the density. As we can see, radius containing $80$ percent of the density generate the best mesh. }
\label{mes_inside}
\end{figure}

\begin{theorem}\label{th_nor}
Classical Gaussian distribution in $\R^3$ can be approximated by \ourdens{} distribution  (with log normal distribution) with parameters:
$$
\mu =\log(8/\pi)-\frac{1}{2}\log(3) \mbox{ and }
\sigma = \log(3\pi/8).
$$
\end{theorem}

\begin{proof}  Observe that both Gaussian and Spherical Log-Normal distributions are spherical. This means that to compare them it is enough to consider the distributions of the radius. In the case of Gaussian in $\R^3$, the distribution of radius is given by $\chi_3$ distribution, which has mean and variance given by 
$$
\begin{array}{rcl}
\mu_{\chi_3} & = &\sqrt{2} \cdot \Gamma[(3 + 1)/2]/\Gamma[3/2]=
2\sqrt{2/\pi}, \\[1ex] 
\sigma^2_{\chi_3} & = & 3-(\mu_{\chi_3})^2=3-\frac{8}{\pi}.
\end{array}
$$
On the other hand, Log-Normal (LN) distribution with parameters $m$ and $\sigma$ has mean and variance given by 
$$
\begin{array}{rcl}
\mu_{LN} & = &\exp(\mu+\sigma^2/2), \\[1ex]
\sigma^2_{LN} & = &(\exp(\sigma^2)-1)\exp(2(\mu+\sigma^2/2)).
\end{array}
$$

Now we have to solve above system of equations and calculate parameters  $\mu,\sigma$
by  $\mu_{\chi_3}$ i $\sigma_{\chi_3}$.

\end{proof}

\subsection{Families of surfaces}

In this section we would like to describe in a more detailed way, how \our{} approximates objects by families of surfaces.
Let us recall that Fig.~3 of the main paper compares how the prior density is modified for the model with Gaussian prior and \ourdens{}. 
For the Gaussian distribution, its tails are transformed into object details, such as wing tips and airplane rear aileron. Therefore, we cannot claim that the peak density models surfaces of the object, while its tails model the registration noise, as is the case for our \ourdens{} distribution.
For \ourdens{}, the distribution tails are spread along object surfaces, modeling registration noise. This allows us to produce the final mesh through the triangulation trick, effectively denoising 3D mesh-based object representation and yielding high-quality results. 
In \our{} we use triangulation on unit sphere. It is motivated by the fact that point on surfaces has symmetric noise (gaussian noise). Nevertheless, we can use triangulation on sphere with different radii (corresponding to different percent of the density). To compare the models, for both of them we can draw the images of spheres which contain inside the same percentage of the data.  In such a case we obtain families of surfaces.
In Fig.~\ref{mes_diff} we present meshes obtained by different radii which contains $20,40,60$ and $80$ percent of the density. \ourdens{} stabilizes triangulation, while for model with normal prior relatively high fluctuations can be observed.

Usually, it is enough to use triangulation on unit sphere. But in some cases we can obtain better meshes by changing radius of the sphere.
For instance, some elements from ShapeNet do not contain only surfaces of objects. In the case of some cars, we have additional elements like steering wheel, see Fig.~\ref{mes_inside}. In such a case, we can use triangulation trick with a larger radius sphere for obtaining better mesh representation, see Fig.~\ref{mes_inside}.

\subsection{Visualization of mesh representation obtained by \our{}}

Below we present:
\begin{itemize}
\item Fig.~\ref{fig:samp}: Mesh representations generated by our \our{} method - extended version of the meshes presented in the main paper.
\item Fig.~\ref{fig:int_m1} and Fig.~\ref{fig:int_m2}: Visualizations of how the triangulation on the sphere is transformed into a mesh of an object.
\item Fig~\ref{fig:sample}: Visualizations on how the samples from our \ourdens{} prior are transformed into points on objects.
\end{itemize}

Overall, our \our{} method offers stable and high-quality object meshes at significantly lower computation cost than the competing point cloud generative models.

\begin{figure}[t!]
\begin{center}
\includegraphics[height=19cm]{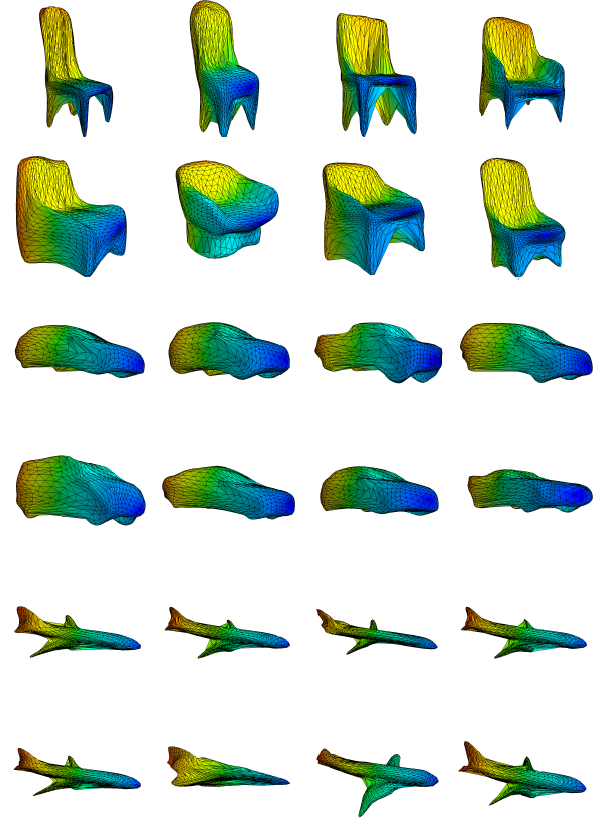}
\end{center} 
\caption{Mesh representations generated by our \our{} method. Contrary to the existing methods that return point cloud representations sparsely distributed in 3D space, our approach allows to create a continuous 3D object representation in the form of high-quality meshes.}
\label{fig:samp}
\vspace{2cm}
\end{figure}

%

\begin{figure}[t!]
\begin{center}
\includegraphics[height=20cm]{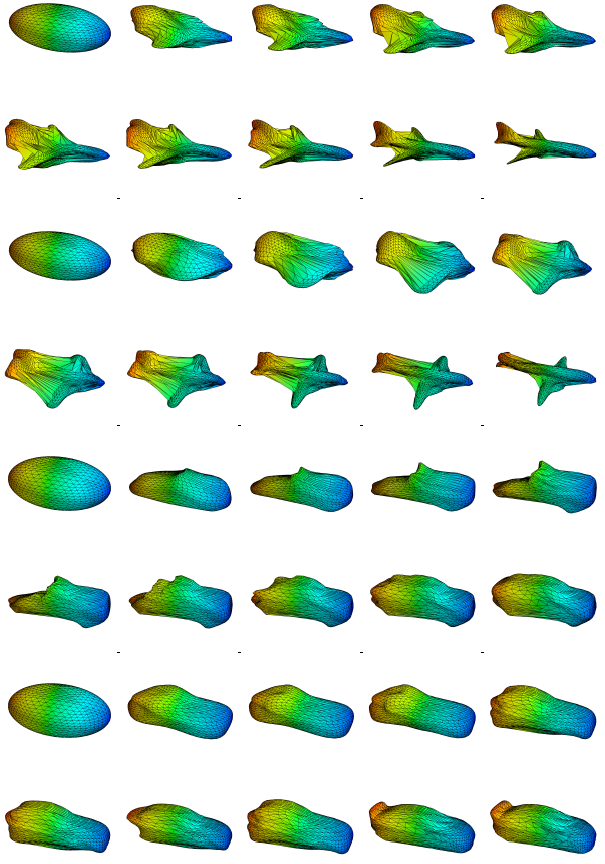}
\end{center} 
\caption{ In the image we present how the triangulation on the sphere is transformed into mesh of object. As we can see, thanks to triangulation trick we obtain high quality mesh. Thanks to us CNF as a target network we can visualize continuous transformation between uniform sphere and surfaces of objects.}
\label{fig:int_m1}
\end{figure}

\begin{figure}[t!]
\begin{center}

\includegraphics[height=20cm]{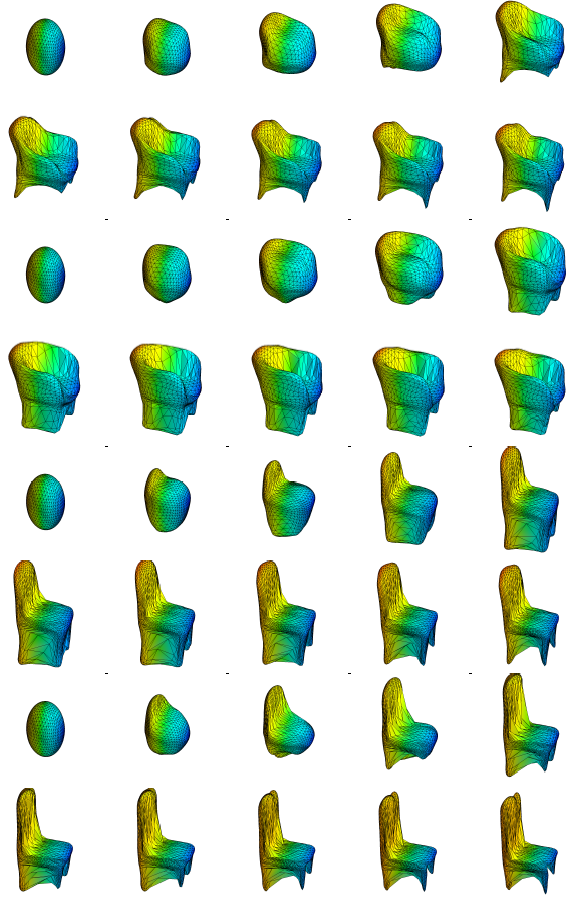}
\end{center} 
\caption{We show how the triangulation on the sphere is transformed into a mesh of object. Thanks to the so-called triangulation trick, we obtain high quality object meshes. Since we use a CNF as a target network, we can visualize a continuous transformation between a uniform sphere and surfaces of objects.}
\label{fig:int_m2}
\end{figure}

\begin{figure}[t!]
\begin{center}
\includegraphics[height=20cm]{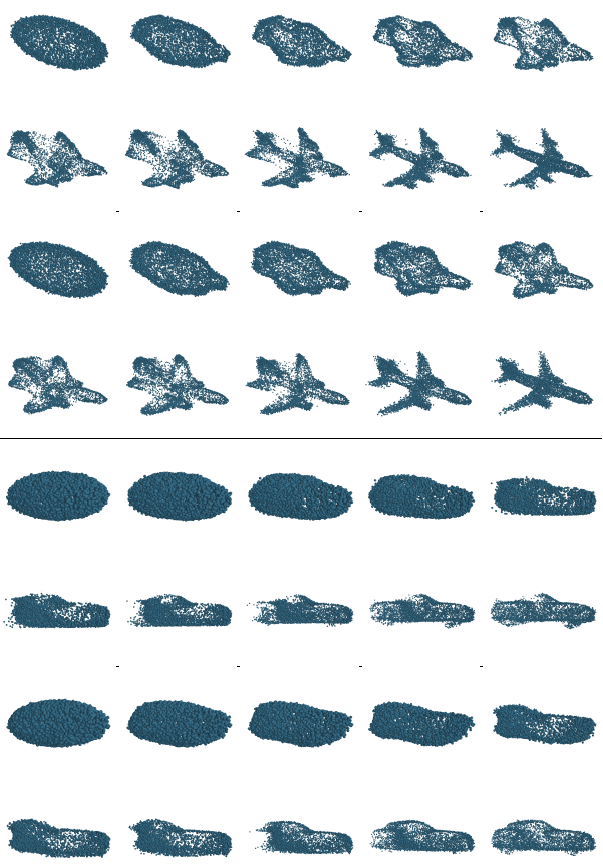}
\end{center} 
\caption{We visualize how the samples from prior (\ourdens{}) are transformed into points on object surfaces. The transformation produces a point representation of a similar quality to PiontFlow.}
\label{fig:sample}
\end{figure}

\bibliographystyle{abbrvnat}

\end{document}